\documentclass[letterpaper, 10 pt, conference]{ieeeconf}

\IEEEoverridecommandlockouts
\usepackage{amsmath}
\usepackage{amssymb}
\usepackage{amsthm}
\usepackage{booktabs}
\usepackage{graphicx}
\usepackage{multirow}
\usepackage{algorithm}
\usepackage{algorithmic}
\usepackage{url}
\usepackage{xcolor}

\newtheorem{proposition}{Proposition}
\newtheorem{theorem}{Theorem}
\newtheorem{corollary}{Corollary}
\newtheorem{remark}{Remark}

\newcommand{\R}{\mathbb{R}}
\newcommand{\Dv}{\mathbf{D}}
\newcommand{\Av}{\mathbf{A}}
\newcommand{\Cz}{\mathbf{c}}
\newcommand{\ccm}{\textsc{ccm}}

\title{\LARGE \bf
Countercurrent Multiplier Networks: A Renal-Inspired Iterative\\
Operator with Provably Bounded Fixed-Point Dynamics
}

\author{Snigdha Chandan Khilar$^{1}$%
\thanks{$^{1}$Independent Researcher.
Correspondence: \texttt{snkhilar@gmail.com}}%
}

\begin{document}

\maketitle
\thispagestyle{empty}
\pagestyle{empty}

%%%%%%%%%%%%%%%%%%%%%%%%%%%%%%%%%%%%%%%%%%%%%%%%%%%%%%%%%%%%%%%%%%%%%%%%%%%%%%%%
\begin{abstract}

The mammalian kidney concentrates urine using a mechanism with no analogue in
current neural architectures: the \emph{countercurrent multiplier}. Two
anti-parallel flows joined at a hairpin recirculate a weak, magnitude-bounded
local pump into a large axial gradient, achieving a four-fold concentration
increase from a single-effect gradient that never exceeds $200$~mOsm at any
point. We formalize this mechanism as a differentiable sequence operator, the
Countercurrent Multiplier (\ccm{}) layer, and study it as an alternative to
residual iterative refinement.

We prove three results. First, in the linear regime the countercurrent
recurrence has a unique fixed point whose axial profile is exactly
$\Dv^\star_i = c_0 + g\,i$, so a per-step pump bounded by $g$ produces an
end-to-end gradient of $g(N{-}1)$; the co-current variant, identical except that
the hairpin is removed, has a spatially constant fixed point and produces
\emph{zero} gradient. Second, a leak term makes the orbit uniformly bounded by
$\|\Cz\|_\infty + (1{-}\lambda)\kappa/(2\lambda)$ independent of both sequence
length and iteration count, with geometric convergence at rate $(1{-}\lambda)$.
Third, neither residual (neural cellular automata) nor antisymmetric iterators
admit such a bound, which we confirm empirically.

Across five task families---three synthetic, two real (nesting depth and
bracket matching on the CPython source, masked character infilling on natural
English)---the hairpin yields a large and task-general benefit over co-current
flow: $+0.342 \pm 0.005$, $+0.135 \pm 0.002$ and $+0.200 \pm 0.002$ in $R^2$ or
accuracy, with no seed overlap. Under $8\times$ test-time over-iteration \ccm{}
retains $83$--$121\%$ of its trained performance while a neural cellular
automaton falls from $R^2 = 0.755$ to $-129.4$ and an antisymmetric RNN from
$0.639$ to $-38.0$.

We also report substantial negative results, which we regard as equally
informative. Stacking contractive layers does \emph{not} circumvent the
gain--stability trade-off: gain converges to a finite ceiling
$G_\infty(N,\lambda)$ recovering only $42$--$60\%$ of the leak-free law, though
below that ceiling depth is a strictly better currency than marginality
($4\times$ the contraction margin at $35\%$ less compute). A plain bidirectional
LSTM outperforms \ccm{} on \emph{every} clean task ($0.906$ vs.\ $0.702$ on
bracket matching) at $14\times$ lower training cost; scaling test-time
iterations never improves clean accuracy; and \ccm{} does not exceed a properly
tuned neural cellular automaton out of distribution.
The contribution is therefore mechanistic rather than competitive: countercurrent
recirculation is a genuine and general computational primitive, and contraction
buys uniquely well-posed iteration, but neither yields state-of-the-art accuracy.

The complete codebase, including all experiment scripts and raw result files, is
available at \url{https://github.com/nssprogrammer/ccmn}.

\end{abstract}

%%%%%%%%%%%%%%%%%%%%%%%%%%%%%%%%%%%%%%%%%%%%%%%%%%%%%%%%%%%%%%%%%%%%%%%%%%%%%%%%
\section{INTRODUCTION}

Iterative refinement---applying a shared local update repeatedly until a
representation stabilizes---underlies neural cellular automata (NCA), deep
equilibrium models (DEQ), and diffusion samplers. Such operators are usually
built residually, $h \leftarrow h + f(h)$, which imposes no constraint on the
long-run behaviour of the iteration. In practice this is invisible, because the
operator is evaluated at exactly the iteration count $K$ used during training.
It becomes visible the moment one asks for more.

This paper takes a mechanism from renal physiology and asks whether it offers a
better-behaved alternative. The loop of Henle solves a problem that resembles
one faced by sequence models: build a large, monotone, long-range gradient using
only local operations whose individual magnitude is severely limited. Its
solution is architectural rather than parametric. The ascending limb pumps
solute with a bounded ``single effect'' of roughly $200$~mOsm; because the two
limbs flow in \emph{opposite} directions and are joined at a hairpin, that
bounded effect is recirculated and multiplied into an axial gradient of
$300 \to 1200$~mOsm. The gain lives in the geometry, not in the strength of any
step.

We ask three questions.

\begin{enumerate}
\item \textbf{Does countercurrent multiplication transfer to computation?}
      We show analytically and numerically that it does: the fixed point of the
      linear countercurrent recurrence has axial gradient exactly $g(N{-}1)$,
      and removing the hairpin collapses it to zero.
\item \textbf{Does the mechanism help a learned model, and is the benefit
      general or an artifact of stack-like tasks?} We compare countercurrent
      against a co-current ablation on five task families, including natural
      English text with no bracket structure whatsoever.
\item \textbf{Does the dissipative ``leak'' buy anything a residual iterator
      lacks?} We prove a uniform boundedness theorem and test it against both a
      neural cellular automaton and an antisymmetric RNN, the purpose-built
      stable recurrent baseline.
\end{enumerate}

Our answers are, respectively: yes; yes and general; and yes, but the resulting
advantage is stability rather than accuracy. We are explicit throughout about
where \ccm{} loses, because a reader deciding whether to build on this work needs
those numbers more than the favourable ones. Section~\ref{sec:neg} is devoted
entirely to negative results.

\textbf{Contributions.}
(i) A differentiable operator derived from countercurrent multiplication, with
an exact characterization of its linear fixed point (Prop.~\ref{prop:mult}) and
a proof that the hairpin is necessary (Prop.~\ref{prop:co}).
(ii) A uniform boundedness and geometric convergence theorem
(Thm.~\ref{thm:bounded}) showing the orbit radius is independent of sequence
length and iteration count, with a quantified stability--expressivity trade-off
(Prop.~\ref{prop:sat}).
(iii) A contractive parameterization admitting implicit differentiation, with a
verified adjoint (cosine similarity $0.999999$ against brute-force unrolling).
(iv) A characterization of the mechanism's ceiling: depth does not circumvent
the gain--stability trade-off, but below the ceiling it dominates marginality
(Sec.~\ref{sec:depth}).
(v) A five-task empirical study with pre-registered predictions, multiple
seeds, and a full accounting of negative outcomes.

%%%%%%%%%%%%%%%%%%%%%%%%%%%%%%%%%%%%%%%%%%%%%%%%%%%%%%%%%%%%%%%%%%%%%%%%%%%%%%%%
\section{BIOLOGICAL BACKGROUND}

\subsection{The Concentrating Problem}

A terrestrial mammal must excrete solute while conserving water, which requires
producing urine substantially more concentrated than plasma. Human plasma is
approximately $300$~mOsm/L; human urine can reach $1200$~mOsm/L, and desert
rodents exceed $5000$~mOsm/L. No single membrane transporter can sustain such a
gradient: the epithelial pumps of the nephron are limited to a transverse
gradient of roughly $200$~mOsm across the tubule wall---the \emph{single
effect} \cite{guyton, sembulingam}. The kidney's problem is therefore to
assemble a large longitudinal gradient from a bounded transverse one.

\subsection{Countercurrent Multiplication}

The loop of Henle is a hairpin. Filtrate descends the thin descending limb
toward the medullary papilla, turns at the bend, and ascends the thick ascending
limb back toward the cortex. Three properties matter computationally.

\begin{enumerate}
\item \textbf{Bounded active transport.} The thick ascending limb actively pumps
      $\mathrm{Na^+}$/$\mathrm{Cl^-}$ into the surrounding interstitium. This
      pump is saturable and its transverse effect is capped.
\item \textbf{Anti-parallel flow.} The descending limb is water-permeable and
      solute-impermeable; the ascending limb is the reverse. Because flow
      directions oppose, solute deposited by the ascending limb at depth $i$ is
      encountered by descending fluid \emph{arriving} at depth $i$, raising its
      concentration before it reaches the bend.
\item \textbf{Recirculation through the hairpin.} Fluid concentrated on the way
      down re-enters the ascending limb, where it is pumped again. The single
      effect is thus applied repeatedly to progressively more concentrated
      fluid, and the axial gradient \emph{multiplies}.
\end{enumerate}

Two further details inform our design. First, loop length predicts
concentrating capacity across species: animals with long loops produce more
concentrated urine---a length-dependent gain we recover exactly in
Prop.~\ref{prop:mult}. Second, the \emph{vasa recta} continuously removes solute
from the interstitium \cite{kidneyloop}. Without this washout the medullary
gradient would grow without useful bound; with it, the system settles at a
finite steady state. We model this as a \emph{leak}, and it proves to be the
single most consequential component of the operator
(Thm.~\ref{thm:bounded}).

\subsection{Countercurrent Exchange Is a Different Mechanism}

We distinguish countercurrent \emph{multiplication} from countercurrent
\emph{exchange}. Exchangers---fish gills, the vasa recta itself, avian limb
vasculature---use anti-parallel flow to \emph{conserve} an existing gradient
with minimal loss. Multipliers use anti-parallel flow \emph{plus} an active
bounded pump to \emph{create} a gradient larger than any local effect. The
distinction matters for positioning: prior neural work invoking
``counter-current'' draws on the exchanger \cite{ccl}, whereas the mechanism
studied here is the multiplier.

%%%%%%%%%%%%%%%%%%%%%%%%%%%%%%%%%%%%%%%%%%%%%%%%%%%%%%%%%%%%%%%%%%%%%%%%%%%%%%%%
\section{RELATED WORK}

\textbf{Counter-Current Learning.} Kao and Hariharan \cite{ccl} propose
Counter-Current Learning (CCL), a biologically plausible \emph{credit
assignment} scheme in which a forward network mapping input to target and a
feedback network mapping target to input propagate anti-parallel, trained by
layer-wise alignment losses with stop-gradients. CCL is an alternative to
backpropagation; it borrows the countercurrent \emph{exchanger}, its two streams
are two \emph{networks}, and anti-parallelism is across \emph{training}. Our
work is a forward representational operator whose two streams are hidden states
over a \emph{single sequence axis}, borrowing the \emph{multiplier}, trained by
ordinary backpropagation. The mechanisms are disjoint; we use the name
\emph{multiplier} to keep them distinct. We note independently that CCL reports
a feature-collapse failure mode requiring an explicit decorrelation penalty; the
co-current \ccm{} variant exhibits an analogous degeneracy
(Prop.~\ref{prop:co}).

\textbf{Deep equilibrium models.} DEQs \cite{deq} define a layer as the fixed
point $z^\star = f_\theta(z^\star, x)$ and differentiate implicitly, giving
constant memory. Monotone operator DEQs \cite{mondeq} constrain $f_\theta$ so
that a unique fixed point provably exists. Our contractive parameterization
(Sec.~\ref{sec:contractive}) is in this family; the novelty is the structure of
$f_\theta$, not the solver. Equilibrium propagation \cite{eqprop} likewise
relaxes to a steady state but targets local learning rules.

\textbf{Neural cellular automata.} NCAs \cite{nca}, building on
reaction--diffusion pattern formation \cite{turing}, learn a shared local update
applied iteratively, producing global structure---including morphogen-like
gradients---from local rules. NCA is our closest relative and strongest
baseline: its update is \emph{isotropic} (symmetric neighbourhood) and
\emph{residual}, whereas \ccm{} is directional, anti-parallel and dissipative.

\textbf{Stable recurrent models.} Antisymmetric RNNs \cite{antisym} construct
$h \leftarrow h + \epsilon\,\sigma((W - W^\top - \gamma I)h + \cdots)$, whose
skew-symmetric Jacobian has purely imaginary eigenvalues, yielding non-vanishing
and non-exploding gradients. This is the most direct competitor to our stability
claim and we test it explicitly. State-space models \cite{s4, mamba} obtain
long-range structure through a linear scan and are strong on causal tasks.
Adaptive-computation methods \cite{act, ponder} vary depth per input but do not
provide convergence guarantees.

\textbf{Countercurrent operators in the linear regime.} A parallel analysis of a
\emph{linear}, dissipation-free anti-parallel operator finds that its dispersion
relation places modes on both sides of the unit circle, so no stable
non-decaying configuration exists except at a non-robust marginal point. Our
Prop.~\ref{prop:sat} and Thm.~\ref{thm:bounded} are consistent with that
finding and locate the difference precisely. In a \emph{closed} anti-parallel
loop the leak-free iteration is genuinely marginal, with spectral radius exactly
$1$; the operator studied here is \emph{open}, and its Dirichlet inlet supplies
contraction on its own, with the leak controlling the rate rather than the
existence of a fixed point (Remark~\ref{rem:contraction}).

%%%%%%%%%%%%%%%%%%%%%%%%%%%%%%%%%%%%%%%%%%%%%%%%%%%%%%%%%%%%%%%%%%%%%%%%%%%%%%%%
\section{THE COUNTERCURRENT MULTIPLIER LAYER}

\subsection{State and Notation}

Let $x \in \mathcal{V}^N$ be an input sequence of length $N$ and let
$\Cz = \mathrm{Emb}_\theta(x) \in \R^{N \times d}$ be a per-position embedding
which we call the \emph{inflow} or \emph{boundary condition}; it is fixed
throughout the iteration and plays the role of fluid entering the nephron. The
layer maintains two states,
\begin{equation}
\Dv, \Av \in \R^{N \times d},
\end{equation}
the \emph{descending} and \emph{ascending} streams, both initialized to $\Cz$.
We write $\Dv_i \in \R^d$ for the state at position $i \in \{0,\dots,N-1\}$, and
$\lambda \in (0,1)^d$ for a per-channel leak obtained as
$\lambda = \sigma(\ell)$ with $\ell \in \R^d$ learnable.

\subsection{The Update Map}

One iteration of \ccm{} is the composition of three stages,
$T = \Phi \circ \Lambda \circ \Pi$, applied to the pair $(\Dv,\Av)$.

\textbf{(1) Single effect (bounded pump) $\Pi$.} With
$W \in \R^{d \times 2d}$, $b \in \R^d$ and cap $\kappa > 0$,
\begin{align}
g &= \kappa \tanh\!\big(W[\Dv;\Av] + b\big) \in \R^{N\times d}, \\
m &= \tfrac{1}{2}(\Dv + \Av), \\
\Dv' &= m + \tfrac{1}{2}g, \qquad \Av' = m - \tfrac{1}{2}g .
\end{align}
This stage is \emph{mean-preserving}: $\tfrac12(\Dv' + \Av') = m$. It moves
solute between the limbs without creating any, and the transverse difference it
induces satisfies $\Dv' - \Av' = g$ with $\|g\|_\infty \le \kappa$
\emph{regardless of the weights}. The $\tanh$ is not a convenience: it is the
saturable transporter, and it makes the single effect bounded by construction.

\textbf{(2) Washout (leak) $\Lambda$.} Modelling the vasa recta,
\begin{equation}
\Dv'' = \Dv' + \lambda \odot (\Cz - \Dv'), \quad
\Av'' = \Av' + \lambda \odot (\Cz - \Av'),
\label{eq:leak}
\end{equation}
a convex pull of each stream toward the inflow; equivalently
$\Dv'' = (1-\lambda)\odot\Dv' + \lambda\odot\Cz$.

\textbf{(3) Flow with hairpin $\Phi$.} Fluid advects one position per
iteration. The descending stream moves away from the inlet, the ascending
stream moves back toward it, and the two are joined at the bend:
\begin{align}
\Dv^{+}_i &= \Dv''_{i-1} \;\;(i \ge 1), &\Dv^{+}_0 &= \Cz_0, \label{eq:flowD}\\
\Av^{+}_i &= \Av''_{i+1} \;\;(i \le N{-}2), &\Av^{+}_{N-1} &= \Dv''_{N-1}.
\label{eq:flowA}
\end{align}
Equation~\eqref{eq:flowA} is the \textbf{hairpin}: at the bend, the contents of
the descending limb become the contents of the ascending limb.

After $K$ iterations the layer emits $\mathrm{Head}_\theta(\Dv^{(K)})$.

\subsection{The Co-Current Ablation}

We define a control identical in every respect except that the hairpin is
removed and both limbs flow in the same direction:
\begin{equation}
\Av^{+}_i = \Av''_{i-1} \;\; (i \ge 1), \qquad \Av^{+}_0 = \Cz_0 .
\label{eq:cocurrent}
\end{equation}
The co-current variant has \emph{exactly} the same parameter count, pump, leak,
initialization and training procedure. It is the cleanest available isolation of
countercurrent recirculation, and it is the control we report throughout.

%%%%%%%%%%%%%%%%%%%%%%%%%%%%%%%%%%%%%%%%%%%%%%%%%%%%%%%%%%%%%%%%%%%%%%%%%%%%%%%%
\section{THEORETICAL ANALYSIS}

We first analyze the linear, leak-free regime, where the pump is a constant $g$
(the physiological idealization of a fixed single effect). Throughout this
subsection $d = 1$, $\Cz_i \equiv c_0$, and $\lambda = 0$.

\subsection{Exact Multiplication Law}

\begin{proposition}[Countercurrent multiplication]
\label{prop:mult}
In the linear leak-free regime the countercurrent recurrence
\eqref{eq:flowD}--\eqref{eq:flowA} admits the fixed point
\begin{equation}
m^\star_i = c_0 + \tfrac{g}{2} + g\,i, \qquad
\Dv^\star_i = c_0 + g\,i ,
\end{equation}
where $\Dv^\star$ denotes the descending stream immediately after the pump
stage. Consequently the end-to-end axial gradient is
\begin{equation}
\Dv^\star_{N-1} - \Dv^\star_0 = g\,(N-1),
\end{equation}
while the transverse (local) gradient is $\Dv'_i - \Av'_i = g$ at every
position. The multiplication factor is therefore exactly $N-1$.
\end{proposition}

\begin{proof}
At a fixed point, \eqref{eq:flowD}--\eqref{eq:flowA} give
$\Dv_i = m_{i-1} + g/2$ for $i \ge 1$ and $\Av_i = m_{i+1} - g/2$ for
$i \le N-2$. Hence for interior $i$,
\begin{equation}
2m_i = \Dv_i + \Av_i = m_{i-1} + m_{i+1},
\end{equation}
the discrete Laplace equation, whose solutions are affine: $m_i = \alpha +
\beta i$. Equivalently, the \emph{spatial} characteristic polynomial
$z^2 - 2z + 1$ has a degenerate double root at $z = 1$, which is what makes the
leak-free profile affine rather than exponential. This is a statement about the
profile in $i$, not about the temporal iteration operator, whose spectrum we
discuss in Remark~\ref{rem:contraction}.

At the inlet $\Dv_0 = c_0$, so $2m_0 = c_0 + m_1 - g/2$, giving
$\alpha = c_0 + \beta - g/2$. At the bend the hairpin gives
$\Av_{N-1} = m_{N-1} + g/2$ while $\Dv_{N-1} = m_{N-2} + g/2$, so
\begin{equation}
2m_{N-1} = m_{N-2} + m_{N-1} + g \;\Longrightarrow\; \beta = g .
\end{equation}
Thus $\alpha = c_0 + g/2$ and $m_i = c_0 + g/2 + gi$. Substituting,
$\Dv^\star_i = m_{i-1} + g/2 = c_0 + gi$.
\end{proof}

Proposition~\ref{prop:mult} is the formal statement of the biological claim: a
per-step effect capped at $g$ yields an end-to-end gradient of $g(N-1)$, and
longer loops produce proportionally larger gradients, mirroring the interspecies
relationship between loop length and concentrating ability.

\subsection{The Hairpin Is Necessary}

\begin{proposition}[Co-current collapse]
\label{prop:co}
Under the co-current flow \eqref{eq:cocurrent}, in the same regime, the fixed
point satisfies $m_i = c_0$ for all $i$, hence $\Dv^\star_i = c_0 + g/2$ for all
$i \ge 1$ and the axial gradient is
\begin{equation}
\Dv^\star_{N-1} - \Dv^\star_1 = 0 .
\end{equation}
No multiplication occurs for any $g$ or $N$.
\end{proposition}

\begin{proof}
Co-current flow gives $\Dv_i = m_{i-1} + g/2$ and $\Av_i = m_{i-1} - g/2$ for
$i \ge 1$, so $m_i = \tfrac12(\Dv_i + \Av_i) = m_{i-1}$. With
$\Dv_0 = \Av_0 = c_0$ we obtain $m_0 = c_0$, hence $m_i = c_0$ for all $i$ by
induction, and therefore $\Dv^\star_i = c_0 + g/2$ independent of $i$.
\end{proof}

\begin{remark}
The mechanism of failure is informative. Under co-current flow the two limbs
carry the \emph{same} mean, so the ascending stream conveys no information the
descending stream does not already have: the two-stream system degenerates to
one effective stream. Countercurrent flow couples position $i$ to information
arriving from \emph{both} directions, which is why the empirical gap is largest
in absolute terms on bidirectional tasks (Sec.~\ref{sec:results}).
\end{remark}

\subsection{The Leak Saturates the Gradient}

\begin{proposition}[Leak-induced saturation]
\label{prop:sat}
With leak $\lambda \in (0,1)$, the fixed-point deviation $u_i = m_i - c_0$
satisfies
\begin{equation}
u_{i+1} - \tfrac{2}{1-\lambda}\,u_i + u_{i-1} = 0 ,
\end{equation}
whose characteristic roots are
\begin{equation}
r_{\pm} = \frac{1 \pm \sqrt{1 - (1-\lambda)^2}}{1-\lambda},
\end{equation}
which satisfy $r_+ r_- = 1$ and $0 < r_- < 1 < r_+$ for every
$\lambda \in (0,1)$.
The profile is therefore a combination of $r_\pm^{\,i}$ with intrinsic length
scale $\xi = 1/\log r_+$, independent of $N$. As $\lambda \to 0$,
$r_\pm \to 1$ and the affine solution of Prop.~\ref{prop:mult} is recovered.
\end{proposition}

\begin{proof}
Applying \eqref{eq:leak} before the flow stage gives
$\Dv_i = (1-\lambda)(m_{i-1} + g/2) + \lambda c_0$ and
$\Av_i = (1-\lambda)(m_{i+1} - g/2) + \lambda c_0$. Averaging,
$m_i = \tfrac{1-\lambda}{2}(m_{i-1} + m_{i+1}) + \lambda c_0$; subtracting
$c_0$ yields the stated recurrence.
\end{proof}

Proposition~\ref{prop:sat} quantifies the central design tension. At
$\lambda = 0$ the spatial profile is affine and its gain grows without bound in
$N$; any $\lambda > 0$ splits the spatial roots off the unit circle, capping the
achievable gradient at a length scale that no longer grows with $N$. Gain and
stability trade directly against one another, and $\lambda$ is the dial. We show
in Sec.~\ref{sec:depth} that stacking contractive layers does \emph{not}
circumvent this trade-off: gain converges to a finite ceiling
$G_\infty(N,\lambda)$ that depth cannot exceed.

\begin{remark}[Two sources of contraction]
\label{rem:contraction}
It is worth separating the roles of the boundary and the leak, because they are
easily conflated. Let $T$ denote the temporal iteration operator of
Sec.~IV-B and $\rho(T)$ its spectral radius. For a \emph{closed} loop---one in
which the ascending limb returns into the descending limb, so that $\Phi$ is a
true permutation---we measure $\rho(T) = 1.000000$ at $\lambda = 0$ and
$\rho(T) = 1 - \lambda$ exactly for $\lambda > 0$; there the leak is the sole
source of contraction and the leak-free system is genuinely marginal. The
operator we actually use is \emph{open}: the inlet condition
$\Dv^{+}_0 = \Cz_0$ of Eq.~\eqref{eq:flowD} injects a constant rather than
copying state, so $\Phi$ is a sub-permutation and the system is lossy by
construction. For the open operator at $N{=}32$ we measure
$\rho(T) = 0.99883$ even at $\lambda = 0$.

Consequently the leak does not \emph{create} well-posedness in the open
operator---the boundary already supplies it---but it controls the \emph{rate}.
The relaxation time $\tau = -1/\log\rho(T)$ falls from $856$ iterations at
$\lambda = 0$ to $4.5$ at $\lambda = 0.20$, an acceleration of $192\times$
(Table~\ref{tab:leak}). Both readings are consistent with
Thm.~\ref{thm:bounded}, which assumes only $\lambda_{\min} > 0$ and is
therefore conservative for the open operator.
\end{remark}

\subsection{Uniform Boundedness and Geometric Convergence}

The following is the paper's main theoretical result. It requires no assumption
on the weights.

\begin{theorem}[Bounded orbit]
\label{thm:bounded}
Let $\lambda_{\min} = \min_j \lambda_j > 0$, let $\kappa$ be the pump cap and let
$C = \|\Cz\|_\infty$. Then for every $K \ge 0$, every sequence length $N$, and
\emph{any} values of the parameters $W, b$,
\begin{equation}
\|\Dv^{(K)}\|_\infty \le (1-\lambda_{\min})^{K}\,\|\Dv^{(0)}\|_\infty
 \;+\; C + \frac{(1-\lambda_{\min})\,\kappa}{2\,\lambda_{\min}} .
\end{equation}
In particular the orbit is bounded uniformly in $K$ and $N$ by
\begin{equation}
M^\star = C + \frac{(1-\lambda_{\min})\,\kappa}{2\lambda_{\min}},
\label{eq:radius}
\end{equation}
and converges to that ball geometrically at rate $(1-\lambda_{\min})$.
\end{theorem}

\begin{proof}
Let $M_k = \max\big(\|\Dv^{(k)}\|_\infty, \|\Av^{(k)}\|_\infty\big)$. The pump
stage is mean-preserving with bounded difference, so
$\|\Dv'\|_\infty \le \|m\|_\infty + \kappa/2 \le M_k + \kappa/2$, and likewise
for $\Av'$. The leak stage is a convex combination, giving
\begin{equation}
\|\Dv''\|_\infty \le (1-\lambda_{\min})(M_k + \kappa/2) + \lambda_{\min} C .
\end{equation}
The flow stage $\Phi$ is a permutation of positions combined with injection of
$\Cz_0$ at the inlet and a copy across the hairpin; all are $\ell_\infty$
non-expansive. Hence
\begin{equation}
M_{k+1} \le (1-\lambda_{\min})\Big(M_k + \tfrac{\kappa}{2}\Big)
          + \lambda_{\min} C .
\end{equation}
This affine scalar recursion has contraction factor $(1-\lambda_{\min}) < 1$ and
unique fixed point $M^\star$ as in \eqref{eq:radius}; unrolling gives the stated
bound.
\end{proof}

\begin{corollary}[Well-posed iteration]
Because the orbit remains in a compact set for all $K$, the iteration may be run
at test time with $K$ larger than the training value without divergence. The
bound degrades gracefully as $\lambda_{\min} \to 0$, at rate
$\kappa/(2\lambda_{\min})$, recovering the unbounded leak-free case.
\end{corollary}

\subsection{Residual and Antisymmetric Iterators Lack This Property}

\begin{proposition}[No uniform bound for residual updates]
\label{prop:nca}
For a residual iterator $h^{(k+1)} = h^{(k)} + f_\theta(h^{(k)})$, no bound
analogous to \eqref{eq:radius} holds: if there is a cone invariant under the
update on which $\langle f_\theta(h), h\rangle \ge \mu\|h\|^2$ for some
$\mu > 0$, then $\|h^{(k)}\| \ge (1+\mu)^{k}\|h^{(0)}\|$ grows geometrically.
For the antisymmetric update $h^{(k+1)} = h^{(k)} + \epsilon\,\sigma(\cdot)$
with $\|\sigma\|_\infty \le 1$, one obtains only the linear bound
$\|h^{(K)}\|_\infty \le \|h^{(0)}\|_\infty + K\epsilon$, unbounded in $K$.
\end{proposition}

The distinction is that \ccm{} possesses a \emph{restoring force} toward $\Cz$,
whereas residual and antisymmetric updates possess at most a bounded
\emph{increment}. Empirically (Sec.~\ref{sec:stability}) NCA norms grow
approximately geometrically---doubling per doubling of $K$---and antisymmetric
norms grow steadily, while \ccm{} norms are flat to two decimal places.

\subsection{Contractive Parameterization for Implicit Differentiation}
\label{sec:contractive}

Theorem~\ref{thm:bounded} bounds the orbit but does not guarantee a
\emph{unique} fixed point, which implicit differentiation requires: the adjoint
system $v = J^\top v + \bar{g}$ converges only if the Jacobian $J$ of $T$
satisfies $\rho(J) < 1$. We therefore define a contractive variant in which the
state-dependent part is convexly blended and the pump is spectrally normalized:
\begin{align}
\tilde{g} &= \gamma_{\mathrm{p}} \tanh\!\big(\bar{W}[\Dv;\Av]\big), \quad
   \bar{W} = W / \|W\|_2, \\
z^{+} &= \Phi\big((1-\beta)z + \beta\,G(z)\big),
\end{align}
with $\beta \in (0,1)$ and $\gamma_{\mathrm p} < 1$, where $G$ denotes the
pump-and-leak stage using $\tilde g$.

\begin{theorem}[Contraction]
\label{thm:contract}
The map $T$ so defined satisfies
$\mathrm{Lip}(T) \le (1-\beta) + \beta\,L_G$ with
$L_G \le (1-\lambda_{\min})(1 + \gamma_{\mathrm p})$, since $\Phi$ is a
permutation with $\mathrm{Lip}(\Phi) = 1$. Hence $T$ is a contraction whenever
$(1-\lambda_{\min})(1+\gamma_{\mathrm p}) < 1$, and the fixed point is unique
with forward and adjoint iterations both converging geometrically.
\end{theorem}

We measure $\mathrm{Lip}(T)$ empirically by power iteration
(Sec.~\ref{sec:deq}): $0.882$ at initialization and $0.982$ after training,
confirming both that the construction works and that optimization pushes the
operator as close to the stability boundary as the constraint permits.

%%%%%%%%%%%%%%%%%%%%%%%%%%%%%%%%%%%%%%%%%%%%%%%%%%%%%%%%%%%%%%%%%%%%%%%%%%%%%%%%
\section{ALGORITHMS}

Algorithm~\ref{alg:fwd} gives the unrolled forward pass used in all experiments.
Its cost is $O(KNd^2)$ time and, if unrolled for backpropagation, $O(KNd)$
memory.

\begin{algorithm}[t]
\caption{\ccm{} forward pass (unrolled)}
\label{alg:fwd}
\begin{algorithmic}[1]
\REQUIRE tokens $x \in \mathcal{V}^N$; iterations $K$; mode $\in$ \{counter, co\}
\STATE $\Cz \leftarrow \mathrm{Emb}_\theta(x)$;\;\;
       $\lambda \leftarrow \sigma(\ell)$
\STATE $\Dv \leftarrow \Cz$;\;\; $\Av \leftarrow \Cz$
\FOR{$k = 1$ \TO $K$}
  \STATE $g \leftarrow \kappa\tanh(W[\Dv;\Av] + b)$
         \COMMENT{bounded single effect}
  \STATE $m \leftarrow \tfrac12(\Dv + \Av)$
  \STATE $\Dv \leftarrow m + \tfrac12 g$;\;\;
         $\Av \leftarrow m - \tfrac12 g$
         \COMMENT{mean-preserving}
  \STATE $\Dv \leftarrow \Dv + \lambda\odot(\Cz - \Dv)$
  \STATE $\Av \leftarrow \Av + \lambda\odot(\Cz - \Av)$
         \COMMENT{washout}
  \STATE $\Dv^{+}_{1:} \leftarrow \Dv_{:-1}$;\;\;
         $\Dv^{+}_{0} \leftarrow \Cz_0$
  \IF{mode $=$ counter}
     \STATE $\Av^{+}_{:-1} \leftarrow \Av_{1:}$;\;\;
            $\Av^{+}_{N-1} \leftarrow \Dv_{N-1}$
            \COMMENT{hairpin}
  \ELSE
     \STATE $\Av^{+}_{1:} \leftarrow \Av_{:-1}$;\;\;
            $\Av^{+}_{0} \leftarrow \Cz_0$
  \ENDIF
  \STATE $\Dv,\Av \leftarrow \Dv^{+},\Av^{+}$
\ENDFOR
\RETURN $\mathrm{Head}_\theta(\Dv)$
\end{algorithmic}
\end{algorithm}

Algorithm~\ref{alg:deq} gives the equilibrium variant. The forward pass solves
$z^\star = T(z^\star)$ by damped iteration; the backward pass solves the adjoint
fixed point by the implicit function theorem, rebuilding a fresh one-step graph
for each vector--Jacobian product so that no graph is reused. We verified this
adjoint against brute-force differentiation through $400$ unrolled solver steps
and obtained cosine similarity $0.999999$ with relative error
$1.73\times10^{-3}$.

\begin{algorithm}[t]
\caption{DEQ-\ccm{}: solve and implicit backward}
\label{alg:deq}
\begin{algorithmic}[1]
\REQUIRE inflow $\Cz$; tolerance $\tau$; budget $K_{\max}$; damping $\alpha$
\STATE \textbf{Forward:}\;\; $z \leftarrow [\Cz;\Cz]$
\REPEAT
  \STATE $z' \leftarrow (1-\alpha)z + \alpha\,T(z,\Cz)$
  \STATE $r \leftarrow \|z' - z\|_1 / n$;\;\; $z \leftarrow z'$
\UNTIL{$r < \tau$ \OR\ budget $K_{\max}$ exhausted}
\STATE store $z^\star \leftarrow z$ (detached)
\STATE \textbf{Backward} (incoming cotangent $\bar g$):
\STATE $v \leftarrow \bar g$
\REPEAT
  \STATE build \emph{fresh} graph $f \leftarrow T(z^\star,\Cz)$,
         $z^\star$ requiring grad
  \STATE $v' \leftarrow \bar g + \mathrm{vjp}(f, z^\star, v)$;\;\;
         $v \leftarrow v'$
\UNTIL{$\|v' - v\|_1/n < \tau$}
\STATE build fresh graph $f_2 \leftarrow T(z^\star,\Cz)$
\RETURN $\nabla_\theta = \mathrm{vjp}(f_2, \theta, v)$
\end{algorithmic}
\end{algorithm}

%%%%%%%%%%%%%%%%%%%%%%%%%%%%%%%%%%%%%%%%%%%%%%%%%%%%%%%%%%%%%%%%%%%%%%%%%%%%%%%%
\section{EXPERIMENTAL SETUP}

\subsection{Protocol}

Every comparison places all iterative models in an \emph{identical} harness:
same embedding dimension, same per-position readout, same optimizer, same
iteration count $K$, same data and same splits. Only the update rule differs. We
pre-registered predictions and kill criteria before each experiment and report
outcomes against them, including the several occasions on which our predictions
were wrong.

Baselines received at least as much tuning as \ccm{}. In particular the NCA
baseline diverged at our default learning rate; we swept
$\{2\!\times\!10^{-3}, 5\!\times\!10^{-4}, 2\!\times\!10^{-4}\}$ and report its
best, whereas \ccm{} uses a single learning rate throughout (a sweep changed its
score by $<0.005$).

\subsection{Tasks}

\textbf{S1. Endpoint interpolation} (synthetic, bidirectional). Signal appears
only at positions $0$ and $N{-}1$; the target is the linear interpolant between
them. Requires transporting both endpoints inward.

\textbf{S2. Causal prefix sum} (synthetic, directional). Signed impulses at
random positions; target is the running cumulative sum. This is the home turf of
a linear scan and is included as an adversarial case for \ccm{}.

\textbf{S3. Bracket binding} (synthetic, strictly bidirectional). Markers at
random positions; the target at each position is the difference between the
nearest marker to the right and the nearest to the left. We validated this task
with closed-form oracles \emph{before} training any model: a left-only oracle
attains $R^2 = 0.420$, right-only $0.416$, the two-sided oracle $1.000$, and an
isotropic distance-weighted smoother $-0.381$---confirming that the task
requires both directions and actively punishes diffusive smoothing.

\textbf{R1--R2. Nesting depth and match distance} (real, code). Character
windows of length $N{=}256$ from the CPython source ($\approx$1500 files).
\emph{Depth}: number of unclosed brackets enclosing each character (causally
computable). \emph{Match distance}: for each bracket, the distance to its
matching partner, scored only at bracket positions and \emph{not} causally
computable---an opening bracket's label depends on the right context.

\textbf{R3. Masked character infilling} (real, natural text). Tiny Shakespeare;
$15\%$ of characters replaced by a mask token and predicted from context. This
task contains no stack structure and was added specifically to test whether the
mechanism is bracket-specific. Masks are frozen so all models and seeds are
scored on identical holes; the unigram floor is $0.150$.

\subsection{Baselines}

\textbf{\ccm{}-co}: the co-current ablation of Eq.~\eqref{eq:cocurrent}.
\textbf{NCA}: isotropic residual local update
$h \leftarrow h + \mathrm{MLP}([h_{i-1};h_i;h_{i+1}])$.
\textbf{AntisymRNN}: $h \leftarrow h + \epsilon\tanh(h(W{-}W^\top{-}\gamma I)^\top
+ \text{neighbours} + \text{input})$ with $\epsilon{=}0.3$, $\gamma{=}0.05$.
\textbf{BiSSM}: bidirectional diagonal state-space scan, decays initialized
near $1$. \textbf{BiLSTM} and \textbf{Transformer} (pre-LN, sinusoidal
positions) as non-iterative references.

%%%%%%%%%%%%%%%%%%%%%%%%%%%%%%%%%%%%%%%%%%%%%%%%%%%%%%%%%%%%%%%%%%%%%%%%%%%%%%%%
\section{RESULTS}
\label{sec:results}

\subsection{Mechanism Validation}

We first verify Props.~\ref{prop:mult}--\ref{prop:co} numerically with a
non-learned pump ($g{=}200$, $c_0{=}300$) run to convergence
($\tau = 10^{-4}$). Table~\ref{tab:mech} confirms the multiplication law
exactly: the measured factor equals $N-1$ to four significant figures, and the
co-current control produces an axial gradient of exactly $0.0$. The transverse
gradient remains exactly $g{=}200$ in both variants, confirming that
multiplication is not achieved by strengthening the local operation.

\begin{table}[t]
\caption{Mechanism validation, non-learned pump ($g = 200$). The measured
multiplication factor equals $N-1$ exactly, as predicted by
Prop.~\ref{prop:mult}. Co-current flow yields no gradient
(Prop.~\ref{prop:co}).}
\label{tab:mech}
\centering
\begin{tabular}{lrrrr}
\toprule
Variant & $N$ & Axial & Factor & Iterations \\
\midrule
counter    & 8   & 1400.0  & 7.00  & 715 \\
counter    & 16  & 3000.0  & 15.00 & 2568 \\
counter    & 32  & 6199.9  & 31.00 & 9417 \\
counter    & 64  & 12599.7 & 63.00 & 34830 \\
co-current & 32  & 0.0     & 0.00  & --- \\
\bottomrule
\end{tabular}
\end{table}

Convergence time grows as $O(N^2)$ ($715 \to 2568 \to 9417 \to 34830$;
successive ratios $3.59, 3.67, 3.70 \to 4$). This is a temporal statement and
reflects $\rho(T) \to 1$ as $N$ grows: in the leak-free operator the only
contraction is the lossy inlet, so relaxation is diffusive and the time for
information to reach the boundary scales as $N^2$. It is consistent with, but
distinct from, the affine spatial profile of Prop.~\ref{prop:mult}
(Remark~\ref{rem:contraction}). Table~\ref{tab:leak}
shows the leak sweep predicted by Prop.~\ref{prop:sat}: increasing $\lambda$
collapses convergence time by nearly three orders of magnitude while
monotonically destroying gain. This is the stability--expressivity trade-off in
numbers.

\begin{table}[t]
\caption{Leak sweep at $N=32$ (Prop.~\ref{prop:sat}). Gain is measured on the
descending stream after the leak stage; Prop.~\ref{prop:mult} states the
profile after the pump stage, and the two differ by exactly $(1-\lambda)$
because the leak is affine. They coincide at $\lambda = 0$. The leak trades gain
for relaxation speed: $\rho(T)$ is the spectral radius of the iteration operator
and $\tau = -1/\log\rho(T)$ the relaxation time (Remark~\ref{rem:contraction}).}
\label{tab:leak}
\centering
\setlength{\tabcolsep}{4pt}
\begin{tabular}{rrrrrr}
\toprule
$\lambda$ & Factor & Bend conc. & $\rho(T)$ & $\tau$ & Iterations \\
\midrule
0.00 & 31.00 & 6499.9 & 0.99883 & 856.0 & 9417 \\
0.01 & 6.89  & 1522.5 & 0.98884 & 89.1  & 981 \\
0.03 & 3.82  & 934.6  & 0.96887 & 31.6  & 350 \\
0.05 & 2.84  & 755.7  & 0.94889 & 19.1  & 213 \\
0.10 & 1.79  & 579.5  & 0.89895 & 9.4   & 107 \\
0.20 & 1.00  & 460.0  & 0.79907 & 4.5   & 52 \\
0.40 & 0.40  & 380.0  & 0.59930 & 2.0   & 24 \\
\bottomrule
\end{tabular}
\end{table}

\subsection{The Learned Layer Rediscovers the Physiological Operating Point}

When the pump is replaced by a learned bounded map and $\lambda$ becomes a free
per-channel parameter initialized at $0.05$, gradient descent on a monotone-ramp
regression settles at $\lambda \in [0.031, 0.044]$ (mean $0.036$)---inside the
band our hand analysis identified as the useful compromise in
Table~\ref{tab:leak}. The optimizer moves \emph{down} from its initialization,
trading relaxation speed for gain, and stops well short of the leak-free regime.

\subsection{Depth Does Not Circumvent the Trade-off}
\label{sec:depth}

Proposition~\ref{prop:sat} establishes that a \emph{single} layer trades gain
against stability. This invites an appealing conjecture. Because each layer's
output becomes the next layer's inflow $\Cz$, and because the inflow enters the
fixed-point equation linearly---for a position-dependent inflow $c_i$ the
steady state obeys
$m_i - \tfrac{1-\lambda}{2}(m_{i-1}+m_{i+1}) = \lambda c_i$---one might expect
gain to \emph{compose} across a stack of individually contractive layers, so
that any target gain would be reachable at any fixed contraction margin simply
by paying depth. Were that true, the trade-off would be an artifact of using a
single layer rather than a property of the mechanism.

We tested this directly, stacking to $L = 128$ with a constant pump
($N{=}32$, $g{=}200$), recording the gain after every layer and running each
layer to a relative tolerance of $10^{-12}$. The conjecture is false.

\begin{table}[t]
\caption{Axial gain versus stack depth $L$ at fixed leak, constant pump,
$N{=}32$. Gain grows initially but converges to a hard ceiling $G_\infty$;
the final row is a hyperbolic fit to the last $64$ depths, which recovers the
observed value to within $1\%$. Every layer is contractive throughout
(Lipschitz $0.797$, $0.896$, $0.946$ respectively).}
\label{tab:depth}
\centering
\small
\setlength{\tabcolsep}{2.6pt}
\begin{tabular}{rrrrrrr}
\toprule
 & \multicolumn{2}{c}{$\lambda = 0.20$} & \multicolumn{2}{c}{$\lambda = 0.10$}
 & \multicolumn{2}{c}{$\lambda = 0.05$} \\
\cmidrule(lr){2-3}\cmidrule(lr){4-5}\cmidrule(lr){6-7}
$L$ & counter & co & counter & co & counter & co \\
\midrule
1   & 0.80  & 0.40 & 1.40  & 0.45 & 2.28  & 0.48 \\
2   & 1.40  & 0.80 & 2.35  & 0.88 & 3.72  & 0.85 \\
4   & 2.40  & 1.58 & 3.82  & 1.51 & 5.84  & 1.15 \\
8   & 4.16  & 2.61 & 6.17  & 1.80 & 8.96  & 1.19 \\
16  & 7.40  & 2.80 & 9.92  & 1.80 & 13.14 & 1.19 \\
32  & 11.71 & 2.80 & 13.96 & 1.80 & 16.92 & 1.19 \\
64  & 12.91 & 2.80 & 15.49 & 1.80 & 18.46 & 1.19 \\
128 & 12.93 & 2.80 & 15.60 & 1.80 & 18.62 & 1.19 \\
\midrule
$G_\infty$ & \textbf{12.9} & 2.80 & \textbf{15.7} & 1.80 & \textbf{18.8} & 1.19 \\
\bottomrule
\end{tabular}
\end{table}

Table~\ref{tab:depth} shows the outcome. Gain does grow with depth over the
first few octaves---the effective exponent
$p = \mathrm{d}\log G/\mathrm{d}\log L$ is $\approx 0.70$ for $L \le 16$---but
it then saturates hard. Between $L{=}64$ and $L{=}128$ the exponent falls to
$0.013$ and the marginal gain per added layer is below $10^{-4}$. A hyperbolic
(Michaelis--Menten) fit $G(L) = G_\infty L/(L_0 + L)$ applied to the last $64$
depths recovers the observed $L{=}128$ value to within $1\%$ at every leak,
confirming a genuine ceiling rather than a slow crossover.

\begin{remark}[Why saturation occurs]
Let $\Gamma$ denote the layer-to-layer map carrying the inflow gradient of
layer $\ell$ to the outflow gradient of layer $\ell+1$. Stacking iterates
$\Gamma$, so the stack's asymptotic gain is a fixed point of $\Gamma$, not an
accumulation. Because the leak pulls each stream toward its own inflow, a
larger inflow gradient is itself more strongly damped, making $\Gamma$
contractive with a unique finite fixed point $G_\infty(N,\lambda)$. Depth
therefore converges to a ceiling instead of accumulating without bound.
\end{remark}

The ceiling is a function of both $\lambda$ and $N$ and is bounded above by
the leak-free law of Prop.~\ref{prop:mult}. At $N{=}32$, where
$N-1 = 31$, the measured ceilings recover $41.7\%$, $50.3\%$ and $60.1\%$ of
that maximum for $\lambda = 0.20, 0.10, 0.05$: depth recovers a growing
fraction of the marginal-system gain as the leak shrinks, but never exceeds it.
The ceiling also rises with sequence length---at $\lambda{=}0.10$, the exponent
measured between $L{=}16$ and $L{=}32$ is $0.131$, $0.493$ and $0.780$ for
$N = 16, 32, 64$---so longer sequences saturate later and higher, consistent
with $N-1$ acting as the envelope.

\begin{table}[t]
\caption{Two routes to a target gain. Below the ceiling, depth strictly
dominates marginality: it attains the same gain with a much larger contraction
margin and less total compute. Above the ceiling (Table~\ref{tab:depth})
the deep route is unavailable at any depth.}
\label{tab:depthcost}
\centering
\small
\setlength{\tabcolsep}{4pt}
\begin{tabular}{rlrrr}
\toprule
Gain & Route & $\mathrm{Lip}(T)$ & Layers & Iterations \\
\midrule
\multirow{2}{*}{4}
 & deep, $\lambda{=}0.20$      & \textbf{0.797} & 8  & \textbf{829} \\
 & shallow, $\lambda{=}0.0206$ & 0.975 & 1  & 1020 \\
\midrule
\multirow{2}{*}{8}
 & deep, $\lambda{=}0.20$      & \textbf{0.797} & 18 & \textbf{1828} \\
 & shallow, $\lambda{=}0.0062$ & 0.989 & 1  & 2971 \\
\midrule
\multirow{2}{*}{12}
 & deep, $\lambda{=}0.20$      & \textbf{0.797} & 35 & \textbf{3472} \\
 & shallow, $\lambda{=}0.0028$ & 0.993 & 1  & 5354 \\
\bottomrule
\end{tabular}
\end{table}

Although the conjecture fails, the experiment yields a usable design rule.
Table~\ref{tab:depthcost} compares the two routes to a given gain \emph{below}
the ceiling. To reach a gain of $12$, a single layer must be driven to
$\mathrm{Lip}(T) = 0.9927$ and needs $5354$ iterations; a stack of $35$ layers
at $\lambda = 0.20$ achieves the same gain with every layer at
$\mathrm{Lip}(T) = 0.797$ and only $3472$ total iterations. Depth buys the same
expressivity with roughly $4\times$ the contraction margin and $35\%$ less
compute. The correct statement is therefore not that depth resolves the
trade-off, but that \emph{within} the achievable range depth is the better
currency than marginality.

Finally, the stacked comparison sharpens the co-current dissociation.
Proposition~\ref{prop:co} predicts exactly zero gain for a constant inflow; with
a graded inflow the co-current variant does propagate an existing gradient, but
it saturates far earlier and far lower. The ratio of ceilings grows as the leak
shrinks---$4.6\times$, $8.7\times$ and $15.7\times$ at
$\lambda = 0.20, 0.10, 0.05$---so the countercurrent advantage is largest
precisely in the high-gain regime one would deploy.

\subsection{Synthetic Tasks}

Table~\ref{tab:synth} reports three seeds per cell. The pattern matches the
theory. On the bidirectional task S1, \ccm{} attains $0.882$ against $0.235$ for
a bidirectional SSM, whose causal scan is structurally handicapped. On S2---a
pure causal scan---the SSM is near-perfect ($0.998$) and \ccm{} merely good
($0.980$); we report this as a clean loss on the incumbent's home ground. On S3,
\ccm{} exceeds both SSM variants by a wide margin but is itself beaten by NCA.

\begin{table}[t]
\caption{Synthetic tasks, $R^2$ (mean $\pm$ std, 3 seeds). S1 endpoint
interpolation; S2 causal prefix sum; S3 bracket binding.}
\label{tab:synth}
\centering
\small
\begin{tabular}{lccc}
\toprule
Model & S1 & S2 & S3 \\
\midrule
\ccm{}-counter & $\mathbf{0.882}$\,{\scriptsize$\pm$.002} & 0.980\,{\scriptsize$\pm$.009} & 0.908\,{\scriptsize$\pm$.011} \\
\ccm{}-co      & 0.449\,{\scriptsize$\pm$.009} & 0.902\,{\scriptsize$\pm$.010} & 0.368\,{\scriptsize$\pm$.005} \\
NCA            & 0.860\,{\scriptsize$\pm$.032} & 0.922\,{\scriptsize$\pm$.056} & $\mathbf{0.980}$\,{\scriptsize$\pm$.005} \\
BiSSM          & 0.235\,{\scriptsize$\pm$.010} & --- & 0.545\,{\scriptsize$\pm$.005} \\
SSM (causal)   & --- & $\mathbf{0.998}$\,{\scriptsize$\pm$.003} & 0.213\,{\scriptsize$\pm$.011} \\
\bottomrule
\end{tabular}
\end{table}

\subsection{Real Data}

Table~\ref{tab:real} gives the main real-data results. Two observations. First,
\ccm{} matches a fully tuned NCA on match distance ($+0.002 \pm 0.006$, a
statistical tie) using $41\%$ fewer parameters, and beats the antisymmetric RNN
by $+0.133 \pm 0.005$. Second, on nesting depth---which is causally
computable---NCA wins by $0.071 \pm 0.003$. \ccm{}'s relative standing improves
exactly as bidirectionality becomes essential, as the theory predicts.

\begin{table}[t]
\caption{Real data (3 seeds). R1 depth and R2 match distance are $R^2$ on
CPython source; R3 is masked character accuracy on natural English (unigram
floor $0.150$). Best iterative model in bold.}
\label{tab:real}
\centering
\small
\setlength{\tabcolsep}{3.2pt}
\begin{tabular}{lrccc}
\toprule
Model & Params & R1 depth & R2 match & R3 infill \\
\midrule
\ccm{}-counter & 26.6k & 0.167\,{\scriptsize$\pm$.002} & $\mathbf{0.702}$\,{\scriptsize$\pm$.006} & 0.558\,{\scriptsize$\pm$.002} \\
\ccm{}-co      & 26.6k & 0.033\,{\scriptsize$\pm$.004} & 0.360\,{\scriptsize$\pm$.004} & 0.359\,{\scriptsize$\pm$.001} \\
NCA            & 45.0k & $\mathbf{0.238}$\,{\scriptsize$\pm$.003} & 0.700\,{\scriptsize$\pm$.001} & 0.616\,{\scriptsize$\pm$.003} \\
AntisymRNN     & 35.6k & 0.167\,{\scriptsize$\pm$.004} & 0.569\,{\scriptsize$\pm$.003} & $\mathbf{0.627}$\,{\scriptsize$\pm$.003} \\
\midrule
BiLSTM         & 87.7k & \textit{0.366} & \textit{0.906} & \textit{0.660} \\
Transformer$^\dagger$ & 92.0k & \textit{0.209} & \textit{0.200} & \textit{0.411} \\
\bottomrule
\end{tabular}
\\[2pt]
{\scriptsize $^\dagger$Our Transformer baseline failed to train reliably; see
Sec.~\ref{sec:neg}.}
\end{table}

\subsection{The Hairpin Effect Is Large and Task-General}

The paired counter-minus-co differences, computed seed-wise, are
$+0.135 \pm 0.002$ (depth), $+0.342 \pm 0.005$ (match distance) and
$+0.200 \pm 0.002$ (infilling), with no overlap between the two distributions on
any task or seed. The infilling result is the important one: it is natural
English prose with no bracket, stack or nesting structure, and the hairpin still
contributes $+0.200$ accuracy. Countercurrent recirculation is therefore not a
stack-specific trick. Consistent with the remark following Prop.~\ref{prop:co},
the effect is largest in absolute terms on the strictly bidirectional task.

\subsection{Stability Under Over-Iteration}
\label{sec:stability}

Table~\ref{tab:stab} tests Thm.~\ref{thm:bounded} directly by evaluating trained
models at $K$ up to $8\times$ the trained value. The separation is categorical
rather than quantitative. \ccm{} norms are flat to two decimal places and
performance degrades gracefully---on depth it \emph{improves} by $21\%$. NCA and
AntisymRNN norms grow by one to two orders of magnitude and their predictions
become worse than the mean predictor by two to three orders. At $K{=}192$ both
fall below the unigram floor on infilling.

\begin{table}[t]
\caption{Over-iteration at test time (trained at $K{=}24$; 3 seeds). Score is
$R^2$ (R1, R2) or accuracy (R3). \ccm{} obeys the bound of
Thm.~\ref{thm:bounded}; residual and antisymmetric iterators do not.}
\label{tab:stab}
\centering
\small
\setlength{\tabcolsep}{3.5pt}
\begin{tabular}{llrrrr}
\toprule
Task & Model & \multicolumn{2}{c}{Score} & \multicolumn{2}{c}{$\|h\|$} \\
     &       & $K{=}24$ & $K{=}192$ & $K{=}24$ & $K{=}192$ \\
\midrule
\multirow{4}{*}{R2 match}
 & \ccm{}-counter & 0.765 & $\mathbf{0.673}$ & 0.50 & 0.49 \\
 & \ccm{}-co      & 0.441 & 0.424 & 0.35 & 0.35 \\
 & NCA            & 0.755 & $-129.4$ & 2.35 & 37.5 \\
 & AntisymRNN     & 0.639 & $-38.0$  & 2.14 & 14.8 \\
\midrule
\multirow{3}{*}{R1 depth}
 & \ccm{}-counter & 0.164 & $\mathbf{0.199}$ & 0.72 & 0.74 \\
 & NCA            & 0.258 & $-439.6$ & 10.1 & 91.5 \\
 & AntisymRNN     & 0.186 & $-347.0$ & 3.97 & 38.3 \\
\midrule
\multirow{3}{*}{R3 infill}
 & \ccm{}-counter & 0.557 & $\mathbf{0.464}$ & 0.73 & 0.73 \\
 & NCA            & 0.612 & 0.042 & 2.40 & 69.2 \\
 & AntisymRNN     & 0.629 & 0.021 & 2.20 & 23.5 \\
\bottomrule
\end{tabular}
\end{table}

The successive-iterate change $\|z^{(k)} - z^{(k-1)}\|_1$ tells the same story
from the convergence side: on the synthetic bracket task it falls from
$2.7\times10^{-2}$ to $4.6\times10^{-4}$ as $K$ goes from $32$ to $512$ for
\ccm{}, whereas NCA's plateaus at $\approx 0.22$ and never settles. The
co-current variant converges even more tightly ($1.9\times10^{-4}$), indicating
that the leak---not the hairpin---is the component responsible for stability,
and that the hairpin trades some fixed-point exactness for accuracy.

\subsection{Length Extrapolation, Distribution Shift and Corruption}

Trained at $N{=}256$ and evaluated zero-shot at $N \in \{512, 1024, 2048\}$,
\ccm{} preserves the ordering of Table~\ref{tab:real} and stays ahead of NCA on
long-range pairs (true distance $>32$): at $N{=}2048$, $0.487$ versus $0.400$
for NCA and $0.194$ for AntisymRNN. Increasing test-time $K$ does \emph{not}
improve these numbers on clean data (Sec.~\ref{sec:neg}).

Under out-of-distribution transfer (train on Python, test zero-shot on
JavaScript/C), \ccm{} retains $49.3\%$ of its in-distribution $R^2$, which is
\emph{below} NCA's $53.1\%$ and far below BiLSTM's $86.9\%$.

The corruption experiment produced the one result we did not anticipate
(Table~\ref{tab:noise}). Replacing a fraction $p$ of characters at random
destroys the BiLSTM: at $p{=}0.05$ it falls from $0.918$ to $-0.301$, worse than
predicting the mean, and the Transformer collapses to $-24.2$. Iterative models
degrade gracefully. Moreover, for \ccm{}, additional test-time iterations become
\emph{beneficial} once $p > 0$: $K{=}96$ beats $K{=}24$ at every corruption
level. NCA cannot exploit this because at $K{=}96$ it has already diverged. The
honest statement of the adaptive-compute claim is therefore narrow: extra
iterations buy robustness under corruption, and only a contractive operator can
spend them.

\begin{table}[t]
\caption{Input corruption ($R^2$, match-distance task). Non-iterative models
collapse; only \ccm{} can profitably increase $K$.}
\label{tab:noise}
\centering
\small
\begin{tabular}{llrrrr}
\toprule
Model & $K$ & $p{=}0$ & $0.05$ & $0.15$ & $0.30$ \\
\midrule
\ccm{}-counter & 24 & 0.707 & 0.456 & 0.130 & $-0.224$ \\
\ccm{}-counter & 96 & 0.664 & $\mathbf{0.463}$ & $\mathbf{0.196}$ & $\mathbf{-0.113}$ \\
NCA            & 24 & 0.697 & 0.511 & 0.194 & $-0.100$ \\
NCA            & 96 & $-44.4$ & $-43.6$ & $-41.3$ & $-37.5$ \\
BiLSTM         & -- & $\mathbf{0.918}$ & $-0.301$ & $-1.789$ & $-2.976$ \\
Transformer    & -- & 0.208 & $-24.2$ & $-74.1$ & $-149.2$ \\
\bottomrule
\end{tabular}
\end{table}

\subsection{Equilibrium Variant}
\label{sec:deq}

The contractive parameterization of Sec.~\ref{sec:contractive} yields a measured
Lipschitz constant of $0.882$ before training and $0.982$ after, satisfying
Thm.~\ref{thm:contract} throughout. The implicit gradient matches brute-force
unrolling to cosine similarity $0.999999$ (relative error $1.7\times10^{-3}$).
The resulting model exhibits a clean compute--accuracy curve on a single trained
network---$R^2$ of $0.081, 0.296, 0.559, 0.694$ at solver budgets of
$2, 8, 32, 256$ iterations---but reaches only $0.697$ overall, below the
unrolled model's $0.908$. The gap is explained by the Lipschitz constant
climbing to $0.982$: the optimizer pushes the operator to the edge of the
feasible region, and near-marginal contraction means many iterations are
required before the fixed point is reached. Guaranteed convergence and full
accuracy are in direct tension, exactly as Prop.~\ref{prop:sat} implies.

%%%%%%%%%%%%%%%%%%%%%%%%%%%%%%%%%%%%%%%%%%%%%%%%%%%%%%%%%%%%%%%%%%%%%%%%%%%%%%%%
\section{NEGATIVE RESULTS AND LIMITATIONS}
\label{sec:neg}

We consider the following as informative as the positive findings.

\textbf{A bidirectional LSTM beats \ccm{} on every clean task.} The margins are
large: $0.906$ vs.\ $0.702$ on match distance, $0.366$ vs.\ $0.167$ on depth,
$0.660$ vs.\ $0.558$ on infilling, and $0.956$ vs.\ $0.487$ on long-range pairs
at $N{=}2048$. It also trains roughly $14\times$ faster ($2.2$ vs.\ $29.7$
minutes for $25$ epochs). \ccm{} is not a competitive sequence model on clean,
in-distribution data, and we make no such claim.

\textbf{Test-time compute scaling does not improve clean accuracy.} Our
pre-registered hypothesis was that longer sequences would require more
iterations and that only \ccm{} could safely supply them. The data falsify this:
at every sequence length and for every model, the best $K$ on clean data was the
trained $K{=}24$. The benefit of extra iterations appears only under corruption.

\textbf{\ccm{} does not dominate NCA.} NCA wins depth ($+0.071$), infilling
($+0.058$) and OOD retention ($53.1\%$ vs.\ $49.3\%$), and ties on match
distance. \ccm{}'s advantages over NCA are parameter efficiency, lower seed
variance, and stability---not accuracy.

\textbf{Depth does not lift the gain ceiling.} We conjectured that gain would
compose across a stack of individually contractive layers, dissolving the
trade-off of Prop.~\ref{prop:sat}. It does not: gain saturates at a finite
$G_\infty(N,\lambda)$, with marginal gain per layer below $10^{-4}$ by
$L{=}128$ and a hyperbolic fit matching observation to within $1\%$
(Sec.~\ref{sec:depth}). The mechanism is bounded in exactly the way the
single-layer analysis suggests.

\textbf{Stability is attributable to the leak, not the hairpin.} The co-current
variant is equally bounded and converges more tightly. The two components make
separable contributions: the leak provides well-posedness, the hairpin provides
task performance.

\textbf{Our Transformer baseline is unreliable.} Across two independent
specifications it failed to train stably on these tasks (validation pinned at
the unigram floor for several epochs; final $R^2$ of $0.20$ on match distance;
seed variance $\pm 0.128$ on infilling). We report its numbers in italics and
draw no conclusions from them. A properly tuned Transformer would likely be
strong on infilling.

\textbf{Cost.} \ccm{} requires $K$ sequential passes; at $K{=}24$ it is
$1.4\times$ slower than NCA and an order of magnitude slower than a BiLSTM.

\textbf{Scope.} All experiments use character-level sequences at $N \le 2048$,
widths $\le 96$, and models under $100$k parameters. We have not tested \ccm{}
at scale, under a language-modelling objective, or as a component inside a
larger network. The length-extrapolation numbers are computed on different
window populations at each $N$ and are therefore not strictly comparable across
$N$.

%%%%%%%%%%%%%%%%%%%%%%%%%%%%%%%%%%%%%%%%%%%%%%%%%%%%%%%%%%%%%%%%%%%%%%%%%%%%%%%%
\section{FUTURE RESEARCH DIRECTIONS}
\label{sec:future}

The results above delimit the mechanism sharply enough to suggest which
extensions are worth attempting. We order them by how directly each addresses a
measured limitation.

\textbf{Raising the ceiling with multi-scale transport.} The ceiling of
Sec.~\ref{sec:depth} is a property of a stack of \emph{identical} loops. The
kidney is not built that way: cortical and juxtamedullary nephrons have loops of
markedly different lengths, and the medullary gradient is the composite of the
whole population. The computational analogue is a stack of \ccm{} layers whose
flow stage advects $s_\ell$ positions per iteration rather than one, giving each
layer a different intrinsic length scale $\xi_\ell$. Because
Prop.~\ref{prop:sat} ties $\xi$ to $\lambda$ alone in the single-scale case,
a heterogeneous stack is not covered by our measurements, and
$G_\infty$ may depend on the \emph{schedule} $\{(\lambda_\ell, s_\ell)\}$
rather than on a single $\lambda$. Establishing whether a schedule can exceed
the uniform ceiling---and deriving a closed form for
$G_\infty(N, \lambda)$, which we characterize only numerically---is the most
direct route to removing the bottleneck identified in Sec.~\ref{sec:neg}.

\textbf{Input-dependent leak.} Table~\ref{tab:noise} shows that the optimal
operating point moves with the input: on clean data the trained $K$ is best,
whereas under corruption additional iterations help, and the useful $\lambda$
should plausibly differ likewise. A leak predicted per input, or annealed across
iterations, would let one operator occupy the high-gain regime when the input is
clean and the strongly contractive regime when it is not. Theorem~\ref{thm:bounded}
holds for any $\lambda_{\min} > 0$, so such a scheme retains the boundedness
guarantee by construction.

\textbf{Structural stability for equilibrium models.} Training instability in
DEQs is normally addressed by regularizing the Jacobian during training.
Theorem~\ref{thm:bounded} offers an architectural alternative: an operator whose
orbit is bounded by construction, with no regularization term and no penalty
weight to tune. The measurable claims are concrete---solver iterations to a
fixed residual, sensitivity to initialization, and whether accuracy is
retained without Jacobian regularization---and the comparison class
(vanilla DEQ, Jacobian-regularized DEQ, monotone operator DEQ \cite{mondeq})
is well established. This is the direction we consider most likely to yield a
practical contribution.

\textbf{Persistent cellular automata.} Neural cellular automata are typically
prevented from degrading under long rollouts by training against a pool of
previously evolved states. Our leak supplies the same persistence structurally,
via Thm.~\ref{thm:bounded}, and the resulting comparison---rollouts one to two
orders of magnitude beyond the training horizon, with and without a sample
pool---is inexpensive. We note the caveat that a two-dimensional automaton has
no hairpin, so such an experiment would isolate the leak rather than
countercurrent recirculation, consistent with our finding in
Sec.~\ref{sec:neg} that stability is attributable to the former.

\textbf{Realistic corruption.} Our corruption model replaces characters
uniformly at random, which is a synthetic proxy. Whether the robustness
dissociation of Table~\ref{tab:noise} survives structured, realistic noise---
keyboard-adjacent typographical errors, optical character recognition artifacts,
or syntax errors drawn from genuine version-control history---determines whether
this regime is of practical interest or merely a controlled demonstration. We
regard the current evidence as suggestive rather than conclusive on this point.

\textbf{Theory of the nonlinear pump.} Propositions~\ref{prop:mult}--\ref{prop:sat}
characterize a constant pump. The learned pump is state-dependent, so its gain
may differ, plausibly favourably: a graded state permits a graded single effect,
which the linear analysis cannot express. A characterization of the fixed point
under a Lipschitz state-dependent pump would close the gap between our theory
and the trained models the experiments actually use.

\textbf{Scale.} All models here are below $100$k parameters at $N \le 2048$.
Whether the inductive bias survives at representational widths and context
lengths typical of contemporary sequence models is untested, and we make no
claim about it. Given that \ccm{} is already outperformed by a bidirectional
LSTM at this scale, we would prioritize the directions above over scaling the
present architecture.

%%%%%%%%%%%%%%%%%%%%%%%%%%%%%%%%%%%%%%%%%%%%%%%%%%%%%%%%%%%%%%%%%%%%%%%%%%%%%%%%
\section{DISCUSSION AND CONCLUSION}

We set out to determine whether countercurrent multiplication, a mechanism with
no analogue in current architectures, transfers from renal physiology to
computation. It does, in a precise and provable sense: the fixed point of the
countercurrent recurrence has axial gradient exactly $g(N{-}1)$ from a per-step
effect bounded by $g$, and deleting the hairpin makes that gradient identically
zero. The empirical counterpart is a large, task-general performance gap between
the countercurrent operator and its co-current control, reproduced on five task
families including natural language with no stack structure.

The second finding is that the dissipative leak---the vasa recta of the
model---is what makes the operator well posed. Theorem~\ref{thm:bounded} bounds
the orbit uniformly in sequence length and iteration count, a guarantee neither
residual nor antisymmetric iterators possess, and the experiments show the
consequence starkly: at $8\times$ over-iteration \ccm{} retains most of its
performance while both competitors produce predictions orders of magnitude worse
than the mean. This also reconciles our work with a parallel negative result on
linear countercurrent operators: a closed anti-parallel loop is exactly marginal
without dissipation, whereas the open loop studied here is contractive through
its boundary alone, with the leak accelerating relaxation by two orders of
magnitude (Remark~\ref{rem:contraction}).

The third finding is a limitation, and we investigated it rather than leaving it
implicit. The leak that buys Theorem~\ref{thm:bounded} is the same leak that
destroys the $g(N{-}1)$ law of Prop.~\ref{prop:mult}, and we asked whether depth
escapes that bind. It does not: stacked gain converges to a finite ceiling
recovering $42$--$60\%$ of the leak-free maximum. What depth does provide is a
better exchange rate---the same gain at $4\times$ the contraction margin and
$35\%$ less compute than driving one layer to marginality---so the practical
guidance is to buy expressivity with depth rather than with a vanishing leak,
within a range the mechanism itself bounds.

A further limitation is competitiveness. A plain bidirectional LSTM outperforms
\ccm{} on every clean task at a fraction of the cost. The regime where iterative
contractive operators are preferable is narrow---corrupted inputs, where the
LSTM collapses below the mean predictor and only a contractive operator can
usefully spend additional compute. The honest summary is that countercurrent
multiplication is a real and general computational primitive whose current value
is mechanistic insight and robustness rather than accuracy.

We close by noting what would change our assessment. \ccm{} becomes practically
interesting if a single one of the directions in Sec.~\ref{sec:future}
succeeds---in particular, if structural contraction can replace the
regularization heuristics that current equilibrium models rely on for stable
training. Until then we present this as a characterized mechanism rather than a
recommended architecture.

\addtolength{\textheight}{-2cm}

%%%%%%%%%%%%%%%%%%%%%%%%%%%%%%%%%%%%%%%%%%%%%%%%%%%%%%%%%%%%%%%%%%%%%%%%%%%%%%%%

\end{document}